\documentclass[a4paper]{article}

\usepackage[utf8x]{inputenc}
\usepackage[T1]{fontenc}

\usepackage[a4paper,margin=1in]{geometry}
\usepackage[colorlinks=true,allcolors=blue,pagebackref]{hyperref}
\usepackage[round]{natbib}
\usepackage{amsfonts,amsmath,amssymb}
\usepackage{mathtools}

\usepackage{amsthm} 
\usepackage{mathtools,thmtools}

\usepackage{enumitem}
\usepackage{bm}
\usepackage{xspace}
\usepackage{nicefrac}

\usepackage[capitalise]{cleveref}
\usepackage{dsfont}

\declaretheoremstyle[
	    spaceabove=\topsep, 
	    spacebelow=\topsep, 
	    bodyfont=\normalfont\itshape,
    ]{theorem}

\declaretheorem[style=theorem,name=Theorem]{theorem}

\declaretheoremstyle[
	    spaceabove=\topsep, 
	    spacebelow=\topsep, 
	    bodyfont=\normalfont,
    ]{definition}

\declaretheoremstyle[
        spaceabove=\topsep, 
        spacebelow=\topsep, 
        bodyfont=\normalfont,
        notefont=\normalfont\bfseries,
        notebraces={}{},
        qed=$\blacksquare$, 
    ]{proofstyle}
\declaretheorem[style=proofstyle,numbered=no,name=Proof]{proof}

\declaretheorem[style=theorem,sibling=theorem,name=Lemma]{lemma}

\declaretheorem[style=theorem,sibling=theorem,name=Observation]{observation}

\declaretheorem[style=theorem,numbered=no,name=Theorem]{theorem*}
\declaretheorem[style=theorem,numbered=no,name=Lemma]{lemma*}
\declaretheorem[style=theorem,numbered=no,name=Corollary]{corollary*}
\declaretheorem[style=theorem,numbered=no,name=Proposition]{proposition*}
\declaretheorem[style=theorem,numbered=no,name=Claim]{claim*}
\declaretheorem[style=theorem,numbered=no,name=Fact]{fact*}
\declaretheorem[style=theorem,numbered=no,name=Observation]{observation*}
\declaretheorem[style=theorem,numbered=no,name=Conjecture]{conjecture*}

\declaretheorem[style=definition,numbered=no,name=Definition]{definition*}
\declaretheorem[style=definition,numbered=no,name=Remark]{remark*}
\declaretheorem[style=definition,numbered=no,name=Example]{example*}
\declaretheorem[style=definition,numbered=no,name=Question]{question*}

\DeclareMathAlphabet{\mathbfsf}{\encodingdefault}{\sfdefault}{bx}{n}

\DeclareMathOperator*{\argmin}{arg\!\min}

\newcommand{\lr}[1]{\mathopen{}\left(#1\right)}
\newcommand{\Lr}[1]{\mathopen{}\big(#1\big)}
\newcommand{\LR}[1]{\mathopen{}\Big(#1\Big)}

\newcommand{\lrbra}[1]{\mathopen{}\left[#1\right]}
\newcommand{\Lrbra}[1]{\mathopen{}\big[#1\big]}
\newcommand{\LRbra}[1]{\mathopen{}\Big[#1\Big]}

\newcommand{\norm}[1]{\|#1\|}

\newcommand{\lrset}[1]{\mathopen{}\left\{#1\right\}}
\newcommand{\Lrset}[1]{\mathopen{}\big\{#1\big\}}

\newcommand{\abs}[1]{|#1|}

\newcommand{\ceil}[1]{\lceil #1 \rceil}

\newcommand{\E}{\mathbb{E}}

\newcommand{\tr}{^{\mkern-1.5mu\scriptstyle\mathsf{T}}}

\newcommand{\st}{\star}

\newcommand{\reals}{\mathbb{R}}

\newcommand{\half}{\frac{1}{2}}
\newcommand{\thalf}{\tfrac{1}{2}}

\newcommand{\eqdef}{\stackrel{\text{def}}{=}}

\let\oldtfrac\tfrac
\renewcommand{\tfrac}[2]{\smash{\oldtfrac{#1}{#2}}}

\let\nablaold\nabla
\renewcommand{\nabla}{\nablaold\mkern-1mu}
\crefname{observation}{Observation}{Observations}

\usepackage{float}
\usepackage{enumitem}

\newcommand{\cR}{\mathcal{R}}

\newcommand{\cW}{\mathcal{W}}

\newcommand{\cO}{\mathcal{O}}

\newcommand{\cS}{\mathcal{S}}

\newcommand{\regret}{\cR}
\newcommand{\pregret}{\overline{\cR}}
\newcommand{\grad}{g}

\newcommand{\hess}{\nabla^2}
\renewcommand{\eqdef}{:=}

\newcommand{\eeq}{\end{aligned}\end{equation}}
\newcommand{\beq}{\begin{equation}\begin{aligned}}

\title{Prediction with Corrupted Expert Advice}

\author{
    Idan Amir\thanks{Department of Electrical Engineering, Tel Aviv University; \texttt{idanamir@mail.tau.ac.il}.} 
    \and Idan Attias\thanks{Department of Computer Science, Ben-Gurion University; \texttt{idanatti@post.bgu.ac.il}.} 
    \and Tomer Koren\thanks{School of Computer Science, Tel Aviv University and Google Research, Tel Aviv; \texttt{tkoren@tauex.tau.ac.il}.} 
    \and Roi Livni\thanks{Department of Electrical Engineering, Tel Aviv University; \texttt{rlivni@tauex.tau.ac.il}.} 
    \and Yishay Mansour\thanks{School of Computer Science, Tel Aviv University and Google Research, Tel Aviv; \texttt{mansour.yishay@gmail.com}.}
}

\begin{document}
\maketitle

\begin{abstract}
We revisit the fundamental problem of prediction with expert advice, in a setting where the environment is benign and generates losses stochastically, but the feedback observed by the learner is subject to a moderate adversarial corruption.
We prove that a variant of the classical Multiplicative Weights algorithm with decreasing step sizes achieves constant regret in this setting and performs optimally in a wide range of environments, regardless of the magnitude of the injected corruption.
Our results reveal a surprising disparity between the often comparable Follow the Regularized Leader (FTRL) and Online Mirror Descent (OMD) frameworks: we show that for experts in the corrupted stochastic regime, the regret performance of OMD is in fact strictly inferior to that of FTRL.
\end{abstract}

\section{Introduction}

Prediction with expert advice is perhaps the single most fundamental problem in online learning and sequential decision making. 
In this problem, the goal of a learner is to aggregate decisions from multiple experts and achieve performance that approaches that of the best individual expert in hindsight. 
The standard performance criterion is the regret: the difference between the loss of the learner and that of the best single expert.
The experts problem is often considered in the so-called adversarial setting, where the losses of the individual experts may be virtually arbitrary and even be chosen by an adversary so as to maximize the learner's regret. 
The canonical algorithm in this setup is the Multiplicative Weights algorithm \citep{littlestone1989weighted,freund1995desicion}, that guarantees an optimal regret of $\smash{ \Theta(\sqrt{T\log{N}}) }$ in any problem with $N$ experts and $T$ decision rounds.

A long line of research in online learning has focused on obtaining better regret guarantees, often referred to as ``fast rates,'' on benign problem instances in which the loss generation process behaves more favourably than in a fully adversarial setup.
A prototypical example of such an instance is the stochastic setting of the experts problem, where the losses of the experts are drawn i.i.d.~over time from a fixed and unknown distribution, and there is a constant gap $\Delta$ between the mean losses of the best and second-best experts.
In this setting, it has been established that the optimal expected regret scales as $\Theta(\log(N)/\Delta)$, and in particular, is bounded by a constant independent of the number of rounds~$T$~\citep{de2014follow,koolen2016combining}.
More recently, \cite{mourtada2019optimality} have shown that this optimal regret is in fact achieved by an adaptive variant of the multiplicative weights algorithm.
Other works have studied various intermediate regimes between stochastic and adversarial, where the challenge is to adapt to the complexity of the problem with little or no prior knowledge (e.g., \cite{cesa2007improved,hazan2010extracting,chiang2012online,rakhlin2013online,koolen2014learning,sani2014exploiting,koolen2015second,van2015fast,erven2011adaptive,van2016metagrad,foster2015adaptive,foster2017zigzag}).

In this work, we consider a different, natural intermediate regime of the experts problem: an adversarially-corrupted stochastic setting.
Here, the adversary can modify the stochastic losses with arbitrary corruptions, as long as the sum of the corruptions is bounded by a parameter $C$, which is unknown to the learner. 
One application domain where corruptions are natural is content/ads recommendation: the presence of malicious users affects the feedback signal received by the learning algorithm, but the objective one cares about is the performance of the system (measured via pseudo regret) on the true population of non-malicious users.
The injection of adversarial corruptions implies that the learner observes losses which are {\em not} distributed i.i.d.~across time steps. 
In principle, one could use the adversarial online learning approach to overcome this challenge, but this will result in significantly inferior regret bounds that scale polynomially with the time horizon.
The challenge is then to extend the favourable constant bounds on the regret achievable in the purely stochastic setting to allow for moderate adversarial corruptions.

In the closely related Multi-Armed Bandit (MAB) partial-information model in online learning, the adversarially-corrupted stochastic setting has recently received considerable attention \citep{lykouris2018stochastic,gupta2019better,zimmert2019optimal,jun2018adversarial,kapoor2019corruption,liu2019data}. 
(Even more recently, a similar setting was also considered in episodic reinforcement learning \cite{lykouris2019corruption}.)
Yet, the natural question of determining the optimal regret rate in the analogous full-information problem remained open.
Given that the optimal bounds in the bandit setting scale linearly with the number of experts (or ``arms'' in the context of MAB), it becomes a fundamental question if this dependence can be reduced to logarithmic with full-information, while preserving the dependence on the other parameters of the problem.

Indeed, our main result shows that the optimal regret in the adversarially-corrupted stochastic setting scales as $\Theta(\log(N)/\Delta + C)$ independently of the horizon $T$, and moreover, this optimal bound is attained by a simple adaptive variant of the classic multiplicative weights algorithm, that does not require knowing the corruption level $C$ in advance.
In fact, it turns out that this simple algorithm performs optimally in all three regimes simultaneously: the pure stochastic setting, the adversarially-corrupted setting, and the fully-adversarial setting.
It is important to note that this kind of behaviour is \emph{not} an immediate consequence of the known $O(\log(N)/\Delta)$ performance in the stochastic case, as presented in \cite{mourtada2019optimality}. Even a small amount of adversarial corruption might hinder the algorithm from rapidly concentrating its decisions on the best (uncorrupted) expert, and in principle, this could potentially have a longer-term effect than just on the $C$ corrupted rounds themselves.

Our strategy for proving these results is based on a novel and delicate analysis of the adaptive multiplicative weights algorithm in the stochastic case, which can be seen as analogous to the approach taken by \cite{wei2018more,zimmert2019optimal} in multi-armed bandits.
The first step in this analysis adapts a standard worst-case regret bound for multiplicative weights with an explicit dependence on the second-moments of the losses to the case of an adaptive step-size sequence. 
Then, we observe that the second-order terms admit a ``self-bounding'' property and their sum can be bounded by the (pseudo-)regret itself.
The other expression in the regret bound, which is a sum of entropy terms that stems from the changing step sizes and captures the stability of the algorithm, is more challenging to handle; we show that this sum is also self-bounded by the regret up to exponentially-decreasing terms that sum up to a constant.
Putting these together lead to a constant regret bound in the stochastic case.
Crucially, since the said arguments are all inherently worst-case and do not directly rely on the i.i.d.~nature of the losses, the whole analysis turns out to be robust to corruptions and yields the additive $C$ term in the moderately corrupted case. 

An interesting byproduct of our analysis is a surprising disparity between two common online learning meta-algorithms: Follow the Regularized Leader (FTRL) and Online Mirror Descent (OMD). 
We show that while both FTRL and OMD give rise to optimal (adaptive) multiplicative weights algorithms in the pure stochastic experts setting,%
\footnote{More precisely, the algorithm derived from OMD achieves a near-optimal (yet still constant, independent of $T$) bound, which is tight up to $\log\log{N}$ and $\log(1/\Delta)$ factors.}
the OMD variant becomes strictly inferior to the FTRL variant once corruptions are introduced, and has a much weaker regret of $\Omega(C/\Delta)$ for a fixed number of experts $N$.
In contrast, the non-adaptive (i.e., fixed step size) variants of the meta-algorithms are well-known to be equivalent in the more general setting of online linear optimization.
We note that a closely related separation result was shown by \cite{orabona2018scale} in the standard adversarial setup, who demonstrated a case where OMD suffers linear regret whereas FTRL guarantees a $\smash{\sqrt{T}}$-type bound. Here, we give a specialized argument in the moderately corrupted setting that reveals a more intricate dependence on the complexity of the problem, in terms of the parameters $C$ and $\Delta$.
We also show a few basic simulations in which this gap is clearly visible and tightly supports our theoretical bounds.

\section{Preliminaries}

\subsection{Problem setup}
We consider the classic problem of prediction with expert advice, with a set of $N$ experts indexed by $[N] = \{1,\dots,N\}$. 
In each time step $t=1,\ldots,T$ the learner chooses a probability vector $p_t=(p_{t,1},\dots,p_{t,N})$ from the simplex $\cS_N = \Lrset{p \in \reals^N : \forall i, \; p_i\geq 0 \; \text{and} \; \sum_{i=1}^N p_i=1 }$.
Thereafter, a loss vector $\ell_t \in [0,1]^N$ is revealed.
We will consider three variants of the problem, as follows.

In the \emph{adversarial} (non-stochastic) setting, the loss vectors $\ell_1,\ldots,\ell_T$ are entirely arbitrary and may be chosen by an adversary.
The goal of the learner is to minimize the regret, given by 
$$
    \regret_T
    \eqdef
    \sum_{t=1}^T p_t \cdot \ell_t - \min_{i \in [N]} \sum_{t=1}^T \ell_{t,i}
    .
$$

In the \emph{stochastic} setting, the loss vectors $\ell_1,\ldots,\ell_T$ are drawn i.i.d.~from a fixed (and unknown) distribution.
We denote the vector of the mean losses by $\E[\ell_{t}] = \mu = (\mu_1,\dots,\mu_N),$ and let $i^\star = \argmin_{i\in [N]} \mu_i$ be the index of the best expert, which we assume is unique. 
The gap between any expert $i$ and best one is denoted $\Delta_i = \mu_i - \mu_{i^\star}$, and we let $\Delta = min_{i \neq i^\star} \lrset{\mu_i - \mu_{i^\star}} > 0$.
The goal of the learner in the stochastic setting is to minimize the \emph{pseudo regret}, defined as
\begin{align} \label{def:pseudo-reg}
    \pregret_T
    \eqdef
    \sum_{t=1}^T p_t\cdot \mu - \sum_{t=1}^T \mu_{i^\star}
    =
    \sum_{t=1}^T \sum_{i=1}^N p_{t,i} \, (\mu_i - \mu_{i^\star})
    .
\end{align}

Finally, in the \emph{adversarially-corrupted stochastic} setting (following \cite{lykouris2018stochastic,gupta2019better}), which is the main focus of this paper, loss vectors $\ell_1,\ldots,\ell_T$ are drawn i.i.d.~from a fixed and unknown distribution as in the stochastic setting with mean rewards $\mu = \E[\ell_t]$, and the same definitions of best expert $i^\star$ and gap $\Delta$.
Subsequently, an adversary is allowed to manipulate the feedback observed by the learner, up to some budget $C>0$ which we refer to as the corruption level. 
Formally, on each round $t=1,\dots,T$:
\begin{enumerate}[nosep,label={(\arabic*)}]
    \item A stochastic loss vector $\ell_t \in [0,1]^N$ is drawn i.i.d.~from a fixed and unknown distribution;
    \item The adversary observes the loss vector $\ell_t$ and generates corrupted losses $\tilde\ell_t \in [0,1]^N$;
    \item The player picks a distribution $p_t \in \cS_N$ over experts, suffers the loss $p_t \cdot \ell_t$, and observes \emph{only the corrupted} loss vector $\tilde{\ell_t}$.
\end{enumerate}
Notice that we allow the adversary to be fully adaptive, in the sense that the corruption on round~$t$ may depend on past choices of the learner (before round $t$) as well as on the realizations of the random loss vectors~$\ell_1,\ldots,\ell_t$ in all rounds up to (and including) round $t$.

We consider the following measure of corruption, which we assume to be unknown to the learner:
\begin{align} \label{eq:corruption-def}
    C 
    = 
    \sum_{t=1}^T \norm{\tilde\ell_t-\ell_t}_\infty
    .
\end{align}
Like in the stochastic setting, the goal of the learner is to minimize the \emph{pseudo regret} (defined in \cref{def:pseudo-reg}). 
Note that, crucially, the pseudo regret of the learner depends only on the (means of) the stochastic losses $\ell_t$ and the adversarial corruption appears only in the feedback observed by the learner.

\subsection{Multiplicative Weights}

We recall two variants of the classic Multiplicative Weights (MW) algorithm that we revisit in this work.
The standard MW algorithm \citep{littlestone1989weighted,freund1995desicion} is parameterized by a fixed step-size parameter $\eta>0$.
For an arbitrary sequence of loss vectors $\grad_1,\ldots,\grad_T \in \reals^N$, it admits the following update rule, on every round~$t$:
\begin{align} \label{eq:mw-classic}
    p_{t,i}
    = 
    \frac{e^{-\eta \sum_{s=1}^{t-1} \grad_{s,i}}}{\sum_{j=1}^N e^{-\eta \sum_{s=1}^{t-1} \grad_{s,j}}}
    ,
    \qquad
    \forall ~ i \in [N]
    .
\end{align}
For the basic, fixed step-size version of our results, we will need a standard second-order regret bound for MW.
\begin{lemma}[\cite{cesa2007improved}; see also \cite{arora2012multiplicative}] \label{lem:mw-2nd-ord}
If $\abs{\grad_{t,i}} \leq 1$ for all $t \geq 1$ and $i \in [N]$, the regret of the MW updates in \cref{eq:mw-classic} is bounded as
\begin{align*}
	\sum_{t=1}^T\sum_{i=1}^N p_{t,i} \lr{ \grad_{t,i} - \grad_{t,i^\star} }
	\leq
	\frac{\log N}{\eta } + \eta \sum_{t=1}^T\sum_{i=1}^N p_{t,i} \grad_{t,i}^2
	.
\end{align*}
\end{lemma}
In particular, the bound implies the well-known $\Theta(\sqrt{T \log{N}})$ optimal regret bound for MW in the adversarial setting, if the step size is properly tuned to $\eta = \Theta(\sqrt{\log(N)/T})$; note that this setting of $\eta$ depends on the time horizon $T$.

An adaptive variant of the MW algorithm that does not require knowledge of $T$ was proposed in \cite{auer2002adaptive}. 
This variant employs a diminishing step size sequence, and takes the form:
\begin{align} \label{eq:mw-ftrl}
    p_{t,i} 
    = 
    \frac{e^{-\eta_t \sum_{s=1}^{t-1} \grad_{s,i}}}{\sum_{j=1}^N e^{-\eta_t \sum_{s=1}^{t-1} \grad_{s,i}}}
    ,
    \qquad
    \forall ~ i \in [N]
    ,
\end{align}
with $\eta_t = \sqrt{\log(N)/t}$ for all $t \geq 1$.
This algorithm was shown to obtain the optimal $\Theta(\sqrt{T \log{N}})$ regret in the adversarial setup for any $T$ \citep{auer2002adaptive,cesa2006prediction}.
We will show that, remarkably, the adaptive MW algorithm also achieves the optimal performance in the adversarially-corrupted experts setting, for any level of corruption.

We remark that the MW algorithm in \cref{eq:mw-ftrl} is in fact an instantiation of the canonical Follow-the-Regularized Leader (FTRL) framework in online optimization with entropy as regularization, when one allows the magnitude of regularization to change from round to round.
MW can also be obtained by instantiating the closely related Online Mirror Descent (OMD) meta-algorithm, that also allows for the regularization to vary across rounds.
(For more background on online optimization, FTRL and OMD, see \cref{sec:online-opt}.)
When the regularization is fixed, it is a well-known fact that the two frameworks are generically equivalent and give rise to precisely the same algorithm, presented in \cref{eq:mw-classic}.
However, when the regularization is time-dependent, they produce different algorithms.
We discuss the disparities between these different variants in more details in \cref{sec:ftrl-vs-omd}.

\section{Main Results}

In this section, we consider the adversarially-corrupted stochastic setting and present our main results. 
As a warm-up, we analyze the Multiplicative Weights algorithm with fixed step sizes while assuming the minimal gap $\Delta$ is known to the learner. 
Then, we consider the general case where neither the gap $\Delta$ nor the corruption level $C$ are known, and prove that the adaptive multiplicative weights algorithm attains optimal performance.

\subsection{A warm-up analysis for known minimal gap} 
\label{sec:fix-rate}

We begin with an easier case where the gap $\Delta$ is known to the learner, and can be used to tune the step size parameter of multiplicative weights (\cref{eq:mw-classic}).
In this case, a fixed step-size algorithm suffices and we have the following.

\begin{theorem} \label{thm:ftrl-stochastic-corrupt-regret-fixed}
The Multiplicative Weights algorithm (\cref{eq:mw-classic}) with $\eta=\Delta/2$ in the adversarially-corrupted stochastic regime with corruption level $C$ over $T$ rounds, achieves constant $\cO(\log(N)/\Delta + C)$ expected pseudo regret.
\end{theorem}

Two basic observations in the analysis are the following.
The first observation gives a straightforward bound on the \emph{corrupted} losses of an expert in terms of its pseudo regret.

\begin{observation} \label{clm:trick1}
For any $t=1,\ldots,T$ and $i\in \lrbra{N}$ the following holds
\begin{align*}
    \Lr{\tilde{\ell}_{t,i}-\tilde{\ell}_{t,i^\star}}^2
    \leq 
    \frac{1}{\Delta}\Lr{\mu_i-\mu_{i^\star}}
    .
\end{align*}
\end{observation}

\begin{proof}
For $i \neq i^\star$, note that $\frac{1}{\Delta}(\mu_i-\mu_i^\star)\geq 1$, and on the other hand, $(\tilde{\ell}_{t,i}-\tilde{\ell}_{t,i^\star})^2 \leq 1$ since $\tilde{\ell}_{t,i}\in \lrbra{0,1}$. 
On the other hand, for $i=i^\star$ we have $\frac{1}{\Delta}(\mu_i-\mu_i^\star) = 0$ and $(\tilde{\ell}_{t,i}-\tilde{\ell}_{t,i^\star})^2 = 0$.
\end{proof}

The second observation relates the regret with respect to the corrupted and uncorrupted losses.

\begin{observation} \label{clm:regret-corupt-bound}
For any probability vectors $\lrset{p_t\in \cS_N : t=1\dots T}$ the following holds
\begin{align*}
    \sum_{t=1}^T \sum_{i=1}^N p_{t,i} \Lr{ \ell_{t,i} -\ell_{t,i^\star} }
    \leq
    \sum_{t=1}^T \sum_{i=1}^N p_{t,i} \Lr{ \tilde{\ell}_{t,i} -\tilde{\ell}_{t,i^\star} }
    +2C
    .
\end{align*}
\end{observation}

\begin{proof}
Denoting $\delta_{t,i}$ as the corruption for expert $i$ at time step $t$, we get
\begin{align*}
    \sum_{t=1}^T \sum_{i=1}^N p_{t,i} \Lr{ \tilde{\ell}_{t,i} -\tilde{\ell}_{t,i^\star} }
    &= 
    \sum_{t=1}^T \sum_{i=1}^N p_{t,i} \lr{ \ell_{t,i} -\ell_{t,i^\star} } + \sum_{t=1}^T \sum_{i=1}^N p_{t,i} \left(\delta_{t,i} -\delta_{t,i^\star}\right)
    .
\end{align*}
By definition of the corruption $\sum_{t=1}^T \max_{i\in \left[N\right]}|\delta_{t,i}| \leq C$ and therefore $\sum_{t=1}^T\sum_{i=1}^N p_{t,i} |\delta_{t,i}| \leq C$. Using the triangle inequality implies that
$
    \sum_{t=1}^T\sum_{i=1}^N p_{t,i} ( \delta_{t,i} - \delta_{t,i^\star} ) 
    \geq 
    -2C
    .
$
\end{proof}

We now turn to prove the theorem.

\begin{proof}[of \cref{thm:ftrl-stochastic-corrupt-regret-fixed}]
We start off with the basic bound of (fixed step size) MW in \cref{lem:mw-2nd-ord}:
\begin{align*}
	\sum_{t=1}^T\sum_{i=1}^N p_{t,i} \Lr{ \tilde{\ell}_{t,i} -\tilde{\ell}_{t,i^\star} }
	\leq
	\frac{\log N}{\eta } + \eta \sum_{t=1}^T\sum_{i=1}^N p_{t,i} \tilde{\ell}_{t,i}^2
	.
\end{align*}
First, note that the regret of playing a fixed sequence $p_1,\dots,p_T$ is not affected by an additive translation of the form $\tilde{\ell}^\prime_{t,i}=\tilde{\ell}_{t,i}-a_t'$ for any constant $a_t$ such that $\tilde{\ell}^\prime_{t,i}\in[-1,1]$. In addition, for the Multiplicative Weights algorithm the sequences $p_1,\dots,p_T$ are also not affected by additive translation. Thus, taking $a_t=\tilde{\ell}_{t,i^\star}$ yields
\begin{align*}
    \sum_{t=1}^T \sum_{i=1}^N p_{t,i} \Lr{ \tilde{\ell}_{t,i} -\tilde{\ell}_{t,i^\star} }
    &\leq
    \frac{\log{N}}{\eta} + \eta \sum_{t=1}^T \sum_{i=1}^N p_{t,i} \Lr{ \tilde{\ell}_{t,i}-\tilde{\ell}_{t,i^\star} }^2
    .
\end{align*}
Applying \cref{clm:trick1,clm:regret-corupt-bound} and rearranging terms implies
\begin{align*}
    \sum_{t=1}^T \sum_{i=1}^N p_{t,i} \Lr{ \ell_{t,i} -\ell_{t,i^\star} }
    \leq
    \frac{\log{N}}{\eta} 
    + 2C
    + \frac{\eta}{\Delta} \sum_{t=1}^T \sum_{i=1}^N p_{t,i}\Lr{ \mu_i-\mu_{i^\star} }
    .
\end{align*}
Taking expectation while using the fact that $p_{t,i}$ and $\ell_{t,i}$ are independent we obtain
\begin{align*}
    \E\Lrbra{\pregret_T}
    &\leq
    \frac{\log{N}}{\eta} + 2C
    + \frac{\eta}{\Delta}\E\Lrbra{\pregret_T}
    .
\end{align*}
Finally, by setting $\eta=\Delta/2$ and rearranging we can conclude that
\begin{align*}
    \E[\pregret_T]
    &\leq 
    \frac{4\log N}{\Delta} + 4C.
    \qedhere
\end{align*}
\end{proof}

\subsection{General analysis with decreasing step sizes}

We now formally state and prove our main result: a constant regret bound in the adversarially-corrupted case for the adaptive MW algorithm (in \cref{eq:mw-ftrl}), that does not require the learner to know neither the gap $\Delta$ nor the corruption level $C$.

\begin{theorem} \label{thm:ftrl-stochastic-corrupt-regret}
The adaptive MW algorithm in \cref{eq:mw-ftrl} with $\eta_t=\sqrt{\log(N)/t}$ in the adversarially-corrupted stochastic regime with corruption level $C$ over $T$ rounds, achieves constant $\cO(\log(N)/\Delta + C)$ expected pseudo regret.
\end{theorem}

Note that this result is tight (up to constants): a lower bound of $\Omega(\log(N)/\Delta)$ was shown by \cite{mourtada2019optimality}, and a lower bound of $\Omega(C)$ is straighforward: consider an instance with $N=2$ experts, means $0$ and $1$ (assigned randomly to the experts) and an adversary that corrupts the first $C$ rounds and assigns a loss of zero to both experts on those rounds; the learner receives no information about the identity of the best expert (whose mean loss is the smallest) during the first $C$ rounds and thus incurs, in expectation, at least $C/2$ pseudo regret over these rounds.

For the proof of \cref{thm:ftrl-stochastic-corrupt-regret} we require two main lemmas.
The first lemma is a second-order regret bound for adaptive MW, analogous to the one stated in \cref{lem:mw-2nd-ord} for the fixed step size case.
Here and throughout the section, we use $H(\cdot)$ to denote the entropy of a probability vector, that is, $H(p) = \sum_{i=1}^N p_i \log(1/p_i)$.

\begin{lemma} \label{lem:ftrl-expert-regret}
For any sequence of loss vectors $\grad_1,\ldots,\grad_T \in \reals^N$, the regret of the adaptive MW algorithm in \cref{eq:mw-ftrl} satisfies
\begin{align*}
	\sum_{t=1}^T\sum_{i=1}^N p_{t,i} \lr{ \grad_{t,i} - \grad_{t,i^\star} }
	\leq
	4\log{N}
	+ \frac{1}{2\log{N}} \sum_{t=1}^T \eta_t H(p_{t+1})
	+ 5 \sum_{t=1}^T \eta_t \sum_{i=1}^N p_{t,i} \grad_{t,i}^2
	,
\end{align*}
provided that $\abs{\grad_{t,i}} \leq 1$ for all $t$ and $i \in [N]$.
\end{lemma}

The lemma is obtained from a more general bound for Follow-the-Regularized Leader, and follows from standard arguments adapted to the case of time-varying regularization. 
For completeness, we give this derivation in \cref{sec:proofs}.
The second lemma, key to our refined analysis of adaptive MW, shows that a properly scaled version of the entropy of any probability vector $p$ is upper bounded by the instantaneous pseudo regret of $p$, up to an exponentially decaying additive term.

\begin{lemma} \label{lemma:entropy-bound}
For any $N>0$, $0<\Delta\leq 1$ and $\tau \geq \tau_0 = 64 \Delta^{-2} \log^2{N}$, we have the following bound for the entropy of any probability vector $p$ and any $i^\star\in [N]$:
\begin{align*}
    \frac{1}{\sqrt{\tau}} H\lr{p}
    &\leq 
    \frac{5}{8} \sum_{i\neq i^\star} p_i \Delta 
        + \frac{2}{\sqrt{\tau}} e^{-\tfrac{1}{8} \Delta \sqrt{\tau}}
    .
\end{align*}
\end{lemma}

We prove the lemma below, but first let us show how it is used to derive our main theorem. 

\begin{proof}[of \cref{thm:ftrl-stochastic-corrupt-regret}]

Applying \cref{lem:ftrl-expert-regret} on the corrupted loss vectors $\grad_t = \tilde{\ell}_t$ and introducing additive translations of $\tilde{\ell}_{t,i^\star}$ as before, yields the bound
\begin{align*}
	\sum_{t=1}^T\sum_{i=1}^N p_{t,i}\Lr{ \tilde{\ell}_{t,i}-\tilde{\ell}_{t,i^\star} }
	\leq
    4\log{N}
	+ \frac{1}{2\log{N}} \sum_{t=1}^T \eta_t H(p_{t+1})
	+ 5 \sum_{t=1}^T\eta_t\sum_{i=1}^N p_{t,i} \Lr{ \tilde{\ell}_{t,i} - \tilde{\ell}_{t,i^\star} }^2
	.
\end{align*}
In \cref{lem:ftrl-2nd-order} (see \cref{sec:ftrl-bounds}) we bound the last term in the bound in terms of the pseudo regret (similarly to the proof of \cref{thm:ftrl-stochastic-corrupt-regret-fixed}), as follows:
\begin{align*}
    \sum_{t=1}^T\eta_t\sum_{i=1}^N p_{t,i} \Lr{ \tilde{\ell}_{t,i} - \tilde{\ell}_{t,i^\star}}^2
    &\leq 
    \frac{16\log N}{\Delta}
    + \frac{1}{8}\pregret_T
    .
\end{align*}
For bounding the first summation in the bound, we use \cref{lemma:entropy-bound}. Summing the lemma's bound over $t=1,\ldots,T$ and bounding the sum of the exponential terms by an integral (refer to \cref{lem:sum-entropy-bound} in \cref{sec:ftrl-bounds} for the details), we obtain
\begin{align*}
    \frac{1}{\log{N}} \sum_{t=1}^T \eta_t H(p_{t+1})
    \leq
    \frac{50\log N}{\Delta}
    +\frac{5}{8} \pregret_T
    .    
\end{align*}
Plugging the two inequalities into the regret bound, we obtain
\begin{align*}
    \sum_{t=1}^T\sum_{i=1}^N p_{t,i}\Lr{ \tilde{\ell}_{t,i}-\tilde{\ell}_{t,i^\star} }
    \leq
    \frac{109\log N}{\Delta}
    + \frac{15}{16}\pregret_T
    .
\end{align*}
Using \cref{clm:regret-corupt-bound} and taking expectation we get
\begin{align*}
    \E\Lrbra{\pregret_T}
    \leq
    \frac{109\log N}{\Delta}
    + 2C
    + \frac{15}{16}\E\Lrbra{\pregret_T}
    .
\end{align*}
Rearranging terms gives the theorem.
\end{proof}

We conclude this section with a proof of our key lemma.

\begin{proof}[of \cref{lemma:entropy-bound}]
We split the analysis of the sum for $i\neq i^\star$ and $i=i^\star$. 
Considering first the case $i=i^\star$, we apply the inequality $\log x \leq x - 1$ for $x\geq 1$ to obtain, for $\tau \geq \tau_0 \geq 64 /\Delta^2$,
\begin{align*}
    \frac{1}{\sqrt{\tau}} p_{i^\star} \log{ \frac{1}{p_{i^\star}} }
    \leq
    \frac{1}{\sqrt{\tau}} \lr{ 1-p_{i^\star} }
    \leq
    \frac{1}{8} \sum_{i\neq i^\star} p_i \Delta
    .
\end{align*}
Next, we examine the remaining terms with $i\neq i^\star$. 
The main idea is to look at two different regimes: one when $p_i>e^{-\frac{1}{2}\Delta\sqrt{\tau}}$ and the other for $p_i\leq e^{-\frac{1}{2}\Delta\sqrt{\tau}}$.
In the former case, we have
\begin{align*}
    \frac{1}{\sqrt{\tau}} p_i \log{ \frac{1}{p_{i}} }
    \leq
    \frac{1}{2\sqrt{\tau}} p_i \Delta \sqrt{\tau}
    =
    \frac{1}{2} p_i\Delta
    .
\end{align*}
For the latter case, we can use the inequality of $\log x \leq 2\sqrt{x}$ for $x > 0$ to obtain 
\begin{align*}
    \frac{1}{\sqrt{\tau}} p_{i}\log{ \frac{1}{p_{i}} }
    \leq
    \frac{2}{\sqrt{\tau}}\sqrt{p_{i}}
    \leq
    \frac{2}{\sqrt{\tau}} e^{-\frac{1}{4}\Delta\sqrt{\tau}}
    .
\end{align*}
Combining both observations for $i\neq i^\star$ implies
\begin{align*}
    \frac{1}{\sqrt{\tau}} \sum_{i\neq i^\star} p_{i}\log{ \frac{1}{p_{i}} }
    \leq
    \frac{1}{2}\sum_{i\neq i^\star}p_i\Delta 
        + \frac{2N}{\sqrt{\tau}}  e^{-\frac{1}{4}\Delta\sqrt{\tau}}
    .
\end{align*}
Finally, note that for $\tau \geq \tau_0 = 64 \log^2(N)/\Delta^2$ it holds that $e^{-\tfrac{1}{4} \Delta \sqrt{\tau}} \leq e^{-\tfrac{1}{8} \Delta \sqrt{\tau} - \log{N}} = N^{-1} e^{-\tfrac{1}{8} \Delta \sqrt{\tau}}$.
This together with our first inequality concludes the proof.
\end{proof}

\subsection{Gap between Follow the Regularized Leader and Online Mirror Descent}
\label{sec:ftrl-vs-omd}

Here we present a surprising contrast between the variants of the adaptive MW algorithm obtained by instantiating the Follow the Regularized Leader (FTRL) and Online Mirror Descent (OMD) meta-algorithms, in the adversarially corrupted regime.
We show that while both give optimal algorithms in the stochastic experts setting, the OMD variant becomes strictly inferior to the FTRL variant once corruptions are introduced.

As remarked above, when the step size (i.e., the magnitude of regularization) is fixed, the two meta-algorithms are equivalent, and produce the classic MW algorithm in \cref{eq:mw-classic} when their regularization is set to the negative entropy function over the probability simplex. (For more background and references, see \cref{sec:online-opt}.)
Once one allows the step-sizes $\eta_t$ to vary across rounds, the FTRL gives the adaptive MW algorithm in \cref{eq:mw-ftrl}, while OMD yields the following updates:
\begin{align} \label{eq:mw-omd}
    p_{t,i} 
    = 
    \frac{e^{-\sum_{s=1}^{t-1} \eta_s \ell_{s,i}}}{\sum_{j=1}^N e^{-\sum_{s=1}^{t-1} \eta_s \ell_{s,i}}}
    ,
    \qquad
    \forall ~ i \in [N]
    .
\end{align}

First, we show that the OMD variant of MW in \cref{eq:mw-omd} obtains the same constant $\cO(\log(N)/\Delta)$ regret bound in the pure stochastic regime, up to small $\log\log{N}$ and $\log(1/\Delta)$ factors. (The proof appears in \cref{sec:omd-stoch}.)

\begin{theorem} \label{thm:omd-expert-bound}
The adaptive MW variant in \cref{eq:mw-omd} with $\eta_t=\sqrt{\log(N)/t}$ in the stochastic regime (with no corruption), achieves constant $\cO(\Delta^{-1}\log{N}\,\log^2(\Delta^{-1} \log N))$ expected pseudo regret for any $T$.
\end{theorem}

On the other hand, we give a simple example which demonstrates that the OMD variant of MW exhibits a strictly inferior performance compared to the FTRL variant (see \cref{eq:mw-ftrl}) when adversarial corruptions are present.
For simplicity, assume that the corruption level $C$ is a positive integer.
Consider the following corrupted stochastic instance with $K=2$ experts.
The mean loss of expert $\#1$ is $\mu_1 = \tfrac12(1-\Delta)$ while the mean loss of expert $\#2$ is $\mu_1 = \tfrac12(1+\Delta)$.
The adversary introduces corruption over the first $C$ rounds, and modifies the first $C$ losses of expert $\#1$ to $1$'s and those of expert $\#2$ to $0$'s.

For this simple problem instance, we show the following (see \cref{sec:omd-lb-proof} for the proof).

\begin{theorem} \label{thm:omd-lb}
The expected pseudo regret of the adaptive MW algorithm in \cref{eq:mw-omd} with $\eta_t=\alpha/\sqrt{t}$ where $\alpha = \Omega(1/\sqrt{C})$ on the instance described above for $T \geq T_1 = \Theta\Lr{\!\min\Lrset{C/\Delta^2, \exp\Lr{\sqrt{C}/\alpha}}}$ rounds is at least~$\Omega(\Delta T_1)$.
\end{theorem}

In particular, if the learner does not have non-trivial bounds on the corruption level $C$ and gap~$\Delta$ (that is, $\alpha$ is a constant independent of $C$ and $\Delta$), then the regret is necessarily at least $\Omega(C/\Delta)$ or is exponentially large in $\sqrt{C}$.

\subsection{Numerical Simulations}
\label{sec:numerical-sim}

We conducted a basic numerical experiment to illustrate our regret bounds and the gap between OMD and FTRL discussed above.
The experiment setup consists of two experts with different gaps $\Delta \in \{0.05,0.15,0.25,0.4\}$. The losses were taken as Bernoullis and the corruption strategy injected contamination in the first rounds up to a total budget of $C$, inflicting maximal loss on the best expert while zeroing the losses of the other expert. 

The results, shown in~\cref{fig:reg_vs_T}, demonstrate that for the stochastic case without corruption ($C=0$) OMD achieves better pseudo regret, but is substantially outperformed by FTRL when $C>0$.
In~\cref{fig:reg_vs_C} we further show the inverse dependence of the pseudo-regret on the minimal gap $\Delta$, which precisely supports our theoretical finding discussed in \cref{sec:ftrl-vs-omd}.

\begin{figure}[ht]
\centering
\includegraphics[width=1\columnwidth]{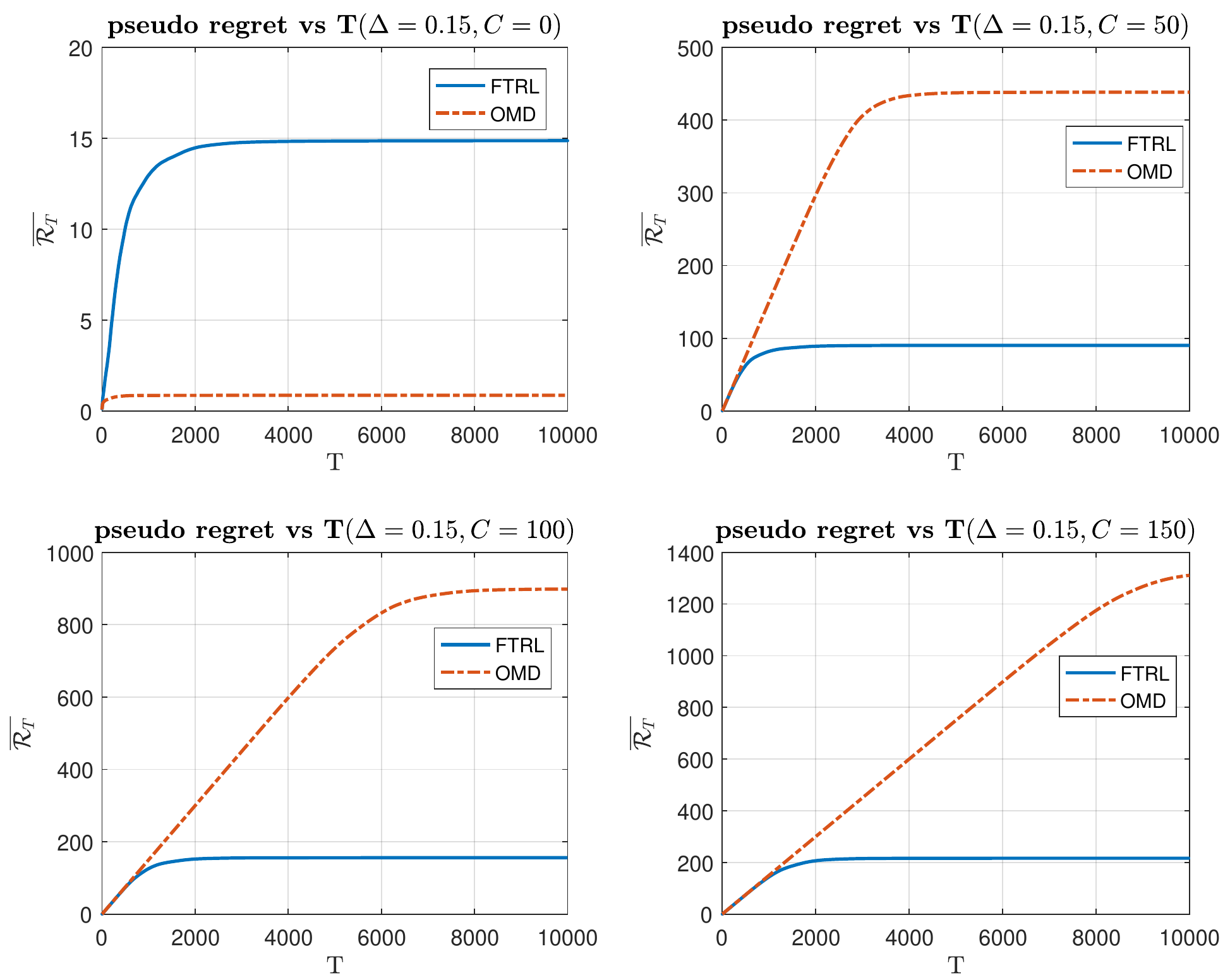}
\caption{Pseudo regret of the two variants of MW, as a function of the number of rounds $T$ for different corruption levels $C$.}
 \label{fig:reg_vs_T}
\end{figure}

\begin{figure}[H]
\centering
\includegraphics[width=1\columnwidth]{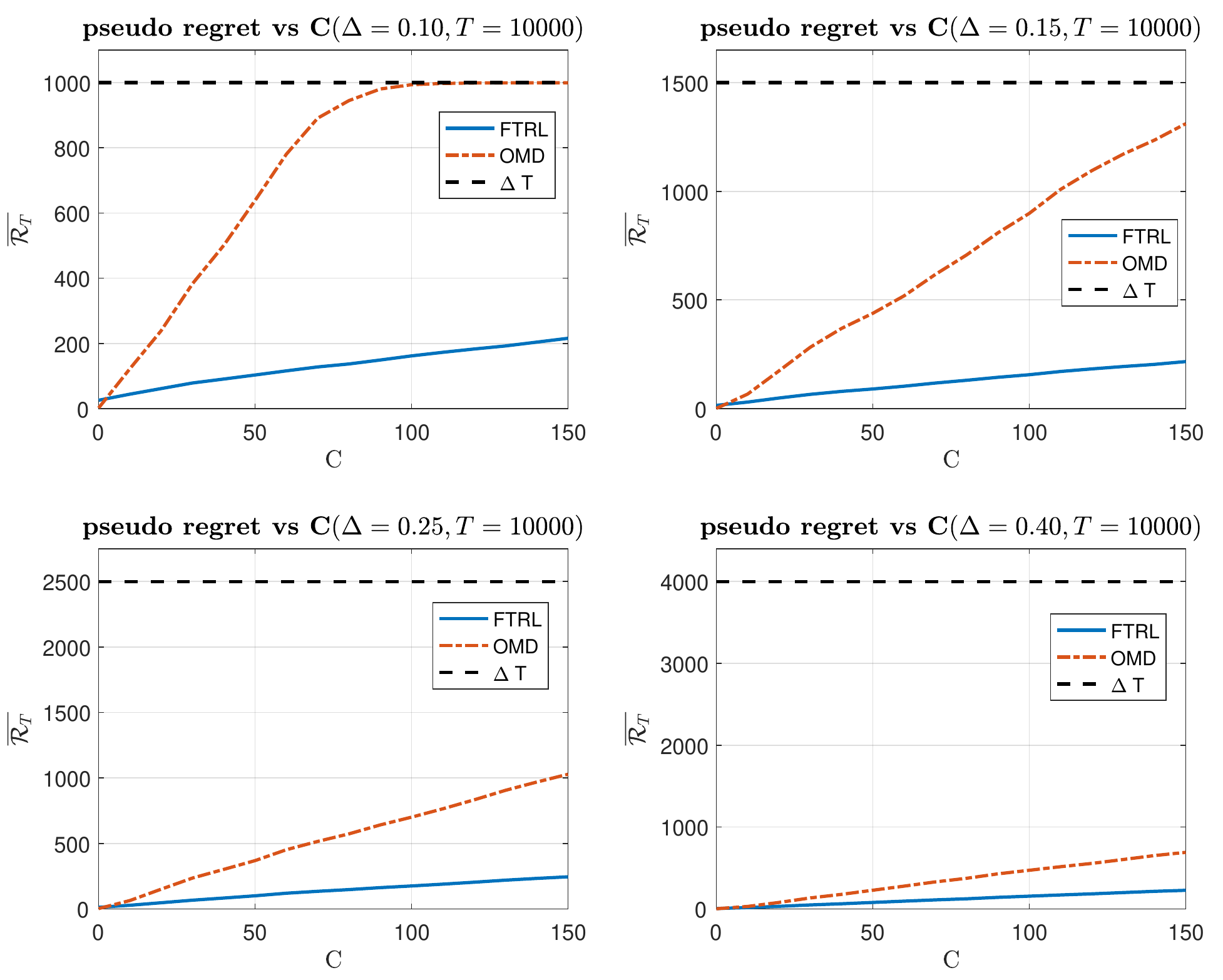}
\caption{Pseudo regret of the two variants of MW, as a function of corruption level $C$ for different values of the gap $\Delta$.}
\label{fig:reg_vs_C}
\end{figure}

\section{Proofs} 
\label{sec:proofs}

\subsection{Preliminaries: Online optimization with time-dependent regularization}
\label{sec:online-opt}

We give a brief background on Follow the Regularized Leader and Online Mirror Descent algorithmic templates, in the case where the regularization is varying and time-dependent.

The setup is the standard setup of online linear optimization.
Let $\cW \subseteq \reals^d$ be a convex domain.
On each prediction round $t=1,\ldots,T$, the learner has to produce a prediction $w_t \in \reals^d$ based on $\grad_1,\ldots,\grad_{t-1}$, and subsequently observes a new loss vector $\grad_t$ and incurs the loss $w_t \cdot \grad_t$.
The goal is to minimize the regret compared to any $w^\star \in \cW$, given by $\sum_{t=1}^T \grad_t \cdot (w_t - w^\star)$.

\paragraph{Follow the Regularized Leader (FTRL).}
The FTRL template generates predictions $w_1,\ldots,w_T \in \cW$, for $t=1,\ldots,T$, as follows:
\begin{align} \label{eq:ftrl-def}
	w_t
	= 
	\argmin_{w \in \cW} \lrset{ w \cdot \sum_{s=1}^{t-1} \grad_s + R_{t}(w) }
	.
\end{align}
Here, $R_1,\ldots,R_T : \cW \to \reals$ is a sequence of twice-differentiable, strictly convex functions.

The derivation and analysis of FTRL-type algorithms is standard; see, e.g., \cite{shalev2012online,OPT-013,orabona2019modern}.
In our analysis, however, we require a particular regret bound that we could not find stated explicitly in the literature (similar bounds exist, however, and date back at least to \cite{duchi2011adaptive}).
For completeness, we provide the bound here with a proof in \cref{sec:online-opt-analysis}. 

\begin{theorem} \label{thm:ftrl-rt}
Suppose that $R_t = \eta_t^{-1} R$ for all $t$ for some strictly convex $R$, with $\eta_1 \geq \ldots \geq \eta_T > 0$.
Then there exists a sequence of points $z_t \in [w_t,w_{t+1}]$ such that the following regret bound holds for all $w^\star \in \cW$:
\begin{align*}
	\sum_{t=1}^T \grad_t \cdot (w_t - w^\star)
	\leq
	\frac{1}{\eta_1} \Lr{ R(w^\star) - R(w_1) }
	+ \sum_{t=1}^T \LR{\frac{1}{\eta_{t+1}}-\frac{1}{\eta_t}} \Lr{ R(w^\star) - R(w_{t+1}) }
	+ \frac{1}{2} \sum_{t=1}^T \eta_t \Lr{\norm{\grad_t}_t^*}^2
	,
\end{align*}
where $\norm{g}_t^2 = g\tr \hess R(z_t) g$ is the local norm induced by $R$ at an appropriate $z_t \in [w_t,w_{t+1}]$, and $\norm{\cdot}_t^*$ is its dual norm.
\end{theorem}

\paragraph{Online Mirror Descent (OMD).}

The closely-related OMD framework produces predictions $w_1,\ldots,w_T$ via the following procedure: initialize $w_{1} = \argmin_{w \in \cW} R_1(w)$, and for $t=1,\ldots,T$, compute
\begin{align} \label{eq:omd-def}
\begin{aligned}
w'_{t+1}
&=
\argmin_{w} \Lrset{ \grad_t \cdot w + D_{R_t}(w,w_t) }
=
(\nabla R_t)^{-1} \Lr{\nabla R_t(w_t)-\grad_t}
;
\\
w_{t+1}
&=
\argmin_{w \in \cW} D_{R_t}(w,w_{t+1}')
.
\end{aligned}
\end{align}
Here, $R_1,\ldots,R_T : \cW \to \reals$ is a sequence of twice-differentiable, strictly convex functions and $D_R(w',w) = R(w') - R(w) - \nabla R(w)\cdot (w'-w)$ is the Bregman divergence of a convex function $R$ at point $w \in \cW$.

The proof of the following regret bound (which is again a somewhat specialized variant of standard bounds for OMD) appears in \cref{sec:online-opt-analysis}.

\begin{theorem}\label{thm:omd-reg}
Suppose that $R_t = \eta_t^{-1} R$ for all $t$ for some strictly convex $R$, with $\eta_1 \geq \ldots \geq \eta_T > 0$.
Then there exists a sequence of points $z_t \in [w_t,w'_{t+1}]$ such that the following regret bound holds for all $w^\star \in \cW$:
\begin{align*}
	\sum_{t=1}^T \grad_t \cdot (w_t - w^\star)
	\leq
	\frac{1}{\eta_1} \Lr{ R(w^\star) - R(w_1) }
	+ \sum_{t=1}^{T-1} \LR{\frac{1}{\eta_{t+1}}-\frac{1}{\eta_t}} D_R(w^\star,w_{t+1})
	+ \frac{1}{2} \sum_{t=1}^T \eta_t \Lr{\norm{\grad_t}_t^*}^2
	,
\end{align*}
where $\norm{\cdot}_t$ is the local norm induced by $R$ at an appropriate $z_t \in [w_t,w'_{t+1}]$, and $\norm{\cdot}_t^*$ is its dual.
\end{theorem}

\subsection{Upper bounds for FTRL}
\label{sec:ftrl-bounds}
\begin{proof} [of \cref{lem:ftrl-expert-regret}]
We observe that \cref{eq:mw-ftrl} is an instantiation of FTRL with $R_t(p) = \eta_t^{-1} R(p)$ as regularizations, where $R(p) = -H(p) = \sum_{i=1}^N p_i \log{p_i}$ is the negative entropy.
Hence, we can invoke \cref{thm:ftrl-rt} to bound the regret compared to any probability distribution $p^\star$.
It suffices to bound the regret for $p^\star$ that minimizes $\sum_{t=1}^T p \cdot \ell_t$, which is always a point-mass on a single expert $i^\star$, for which $R(p^\star) = 0$.
Therefore, \cref{thm:ftrl-rt} in our case reads
\begin{align*}
	\sum_{t=1}^T \sum_{i=1}^N p_{t,i} \Lr{g_{t,i} - g_{t,i^\star}}
	\leq
	-\frac{1}{\eta_1} R(p_1)
	- \sum_{t=1}^T \LR{\frac{1}{\eta_{t+1}}-\frac{1}{\eta_t}} R(p_{t+1})
	+ \frac12 \sum_{t=1}^T \eta_t \Lr{\norm{\grad_t}_t^*}^2
	.
\end{align*}
Now set $\eta_t = \sqrt{\log(N)/t}$.
For the first two terms in the bound, observe that $R(p_{1}) = -\log{N}$, and further, that
\beq \label{eq:geom-eta}
	\frac{1}{\eta_{t+1}}-\frac{1}{\eta_t}
	=
	\frac{1}{\sqrt{\log{N}}} \frac{1}{\sqrt{t} + \sqrt{t+1}}
	\leq
	\frac{1}{2\sqrt{t\log{N}}}
	=
	\frac{\eta_t}{2\log{N}}
	.
\eeq
For the final sum, we have to evaluate the Hessian $\hess R(p'_t)$ at a point $p'_t \in [p_t,p_{t+1}]$.
A straightforward differentiation shows that this matrix is diagonal, with diagonal elements $\hess R(p'_t)_{ii} = 1/p'_{t,i}$. 
Thus,
\beq \label{eq:dual-norm}
	\Lr{\norm{\grad_t}_t^*}^2
	=
	\grad_t\tr \Lr{\hess R(p'_t)}^{-1} \grad_t
	=
	p'_t \cdot \grad_t^2
	.
\eeq
The final sum can be divided and bounded as follows
\begin{align*}
    \sum_{t=1}^T\eta_t \Lr{ p'_t\cdot \grad_t^2 }
    &=
    \sum_{t=1}^{4\log{N}}\eta_t \Lr{ p'_t\cdot \grad_t^2 }
    + \sum_{t=1+4\log{N}}^T\eta_t \Lr{ p'_t\cdot \grad_t^2 }\\
    &\leq 
    4\log{N}
    + \sum_{t=1+\log{N}}^T\eta_t \Lr{ p'_t\cdot \grad_t^2 }
    .
\end{align*}
Where we used the fact that $\sum_{s=1}^t \eta_s = \sum_{s=1}^t \sqrt{\log(N)/s} \leq 2\sqrt{t\log{N}}$.
To conclude the proof it suffices to show that $p'_{t,i} \leq 9 p_{t,i}$ for $t\geq 4\log{N}$.
To see this, denote $G_t = \sum_{s=1}^{t-1} \grad_s$ and write
\begin{align*}
	\frac{e^{-\eta_{t+1} G_{t+1,i}}}{e^{-\eta_t G_{t,i}}}
	=
	e^{-\eta_{t+1} \grad_{t,i}} \, e^{(\eta_t-\eta_{t+1}) G_{t,i}} 
	.
\end{align*}
For $t \geq 4\log{N}$, the following relations hold:
\begin{align*}
	0 &< \eta_{t+1} \abs{\grad_{t,i}} \leq \eta_{t+1} \leq \frac{1}{2}
	;
	\\
	0 &< (\eta_t-\eta_{t+1}) \abs{G_{t,i}}
	\leq \sqrt{\log{N}} \frac{\sqrt{t+1}-\sqrt{t}}{\sqrt{t(t+1)}} t
	\leq \frac{\sqrt{\log{N}}}{\sqrt{t}+\sqrt{t+1}} 
	\leq \eta_t
	\leq \frac{1}{2}
	.
\end{align*}
Hence, for $t \geq 4\log{N}$ we have
\begin{align*}
	\frac{1}{3} 
	\leq 
	\frac{e^{-\eta_{t+1} G_{t+1,i}}}{e^{-\eta_t G_{t,i}}} 
	\leq 
	3,
\end{align*}
and consequently
\begin{align*}
	p_{t+1,i}
	=
	\frac{ e^{-\eta_{t+1} G_{t+1,i}} }{ \sum_{j=1}^N e^{-\eta_{t+1} G_{t+1,j}} }
	\leq
	9 \frac{ e^{-\eta_{t} G_{t,i}} }{ \sum_{j=1}^N e^{-\eta_{t} G_{t,j}} }
	=
	9 p_{t,i}
	.
\end{align*}
Since $p'_t \in [p_t,p_{t+1}]$, the same inequality holds for $p'_t$; that is, $p'_{t,i} \leq 9 p_{t,i}$ for all $i$, and the proof is complete.
\end{proof}

\begin{lemma} \label{lem:ftrl-2nd-order}
For the adaptive MW algorithm in \cref{eq:mw-ftrl} with loss vectors $\grad_t = \tilde\ell_{t,i}$, we have
\begin{align*}
    \sum_{t=1}^T\eta_t\sum_{i=1}^N p_{t,i} \Lr{ \tilde{\ell}_{t,i} - \tilde{\ell}_{t,i^\star}}^2
    &\leq 
    \frac{16\log N}{\Delta}
    + \frac{1}{8}\pregret_T
    .
\end{align*}
\end{lemma}

\begin{proof}
By setting $t_0 = 64\Delta^{-2} \log N$ and $\eta_t=\sqrt{\log(N)/t}$ we obtain
\begin{align*}
    \sum_{t=1}^T\eta_t\sum_{i=1}^N p_{t,i} \Lr{ \tilde{\ell}_{t,i} - \tilde{\ell}_{t,i^\star}}^2
    &\leq
    \sum_{t=1}^{t_0}\eta_t
    + \sum_{t=t_0+1}^{T}\eta_{t_0}\sum_{i=1}^N p_{t,i} \Lr{ \tilde{\ell}_{t,i} - \tilde{\ell}_{t,i^\star} }^2 \\
    &\leq
    2\sqrt{\log(N)}\sqrt{t_0}
    +\frac{\Delta}{8} \sum_{t=t_0+1}^{T}\sum_{i=1}^N p_{t,i} \Lr{ \tilde{\ell}_{t,i} - \tilde{\ell}_{t,i^\star} }^2 \\
    &\leq 
    \frac{16\log N}{\Delta}
    + \frac{1}{8}\sum_{t=t_0+1}^T\sum_{i=1}^N p_{t,i}\Lr{\mu_i-\mu_{i^\star}}
    ,
\end{align*}
where in the final inequality we used \cref{clm:trick1}. 
To conclude we note that $p_{t,i}(\mu_i-\mu_{i^\star})\geq 0$, thus we can modify the last summation to range over $t=1,\ldots,T$.
\end{proof}

\begin{lemma} \label{lem:sum-entropy-bound}
For the adaptive MW algorithm in \cref{eq:mw-ftrl}, we have
\begin{align*}
    \frac{1}{\log{N}} \sum_{t=1}^T \eta_t H(p_{t+1})
    \leq
    \frac{50\log N}{\Delta}
    +\frac{5}{8} \pregret_T
    .    
\end{align*}
\end{lemma}

\begin{proof}
First we split the sum as follows, 
\begin{align*}
    \frac{1}{\log{N}} \sum_{t=1}^T \eta_t H(p_{t+1})
    =
    \frac{1}{\log{N}} \sum_{t=1}^{t_0} \eta_t H(p_{t+1})
    + \frac{1}{\log{N}} \sum_{t=t_0+1}^{T} \eta_t H(p_{t+1})
    ,
\end{align*}
where $t_0 = 64 \Delta^{-2} \log N$.
For the summation of $t=\lrset{t_0+1,\dots,T}$ we use Lemma \ref{lemma:entropy-bound} with $\tau=t\log N\geq t_0\log N= 64 \Delta^{-2} \log^2{N}$ to obtain
\begin{align*}
    \frac{1}{\log{N}} \sum_{t=t_0+1}^T \eta_t H(p_{t+1})
    &=
    \sum_{t=t_0+1}^T\frac{1}{\sqrt{t\log N}}\sum_{i=1}^Np_{t+1,i}\log \frac{1}{p_{t+1,i}} \\
    &\leq 
    \frac{5}{8}\sum_{t=t_0+1}^T\sum_{i\neq i^\star} p_{t+1,i} \Delta
    + 2\sum_{t=t_0+1}^T\frac{1}{\sqrt{t\log N}}e^{-\frac{1}{8}\Delta\sqrt{t \log N}} \\
    &\leq
    \frac{5}{8}\sum_{t=t_0+1}^T\sum_{i=1}^N p_{t,i} \Lr{\mu_i-\mu_{i^\star}}
    + \Delta
    + 2\sum_{t=t_0+1}^T\frac{1}{\sqrt{t\log N}}e^{-\frac{1}{8}\Delta\sqrt{t \log N}}
    ,
\end{align*}
where the last inequality follows for reordering terms in the summation and that $\Delta \leq \mu_i - \mu_{i^\star}$ for $i\neq i^\star$. Using the fact that $p_{t,i}(\mu_i-\mu_{i^\star})\geq 0$ we get
\beq \label{eq:entropy-bound-1}
    \frac{1}{\log{N}} \sum_{t=t_0+1}^T \eta_t H(p_{t+1})
    &\leq
    \frac{5}{8}\sum_{t=1}^T\sum_{i=1}^N p_{t,i} \Lr{\mu_i-\mu_{i^\star}}
    + \Delta
    + 2\sum_{t=t_0+1}^T\frac{1}{\sqrt{t\log N}}e^{-\frac{1}{8}\Delta\sqrt{t \log N}}
    .
\eeq
Moreover, we have
\beq \label{eq:entropy-bound-2}
    \sum_{t=t_0+1}^T\frac{1}{\sqrt{t\log N}}e^{-\frac{1}{8}\Delta\sqrt{\log N}\sqrt{t}}
    &\leq
    \frac{1}{\sqrt{\log N}}\int_{t_0}^{T}\frac{1}{\sqrt{t}}e^{-\frac{1}{8}\Delta\sqrt{\log N}\sqrt{t}}dt \\
    &=
    \frac{1}{\sqrt{\log N}}\cdot\frac{16}{\Delta\sqrt{\log N}}e^{-\frac{1}{8}\Delta\sqrt{\log N}\sqrt{t}}\Big|_T^{t_0} \\
    &\leq
    \frac{16}{\Delta\log N} \\
    &\leq
    \frac{16}{\Delta}
    .
\eeq
Lastly, for the summation of $t=\lrset{1,\dots,t_0}$ we get
\beq \label{eq:entropy-bound-3}
    \frac{1}{\log{N}} \sum_{t=1}^{t_0} \eta_t H(p_{t+1})
    &\leq
    2\sqrt{t_0\log N}
    =
    \frac{16\log N}{\Delta}
\eeq
which follows from $H\lr{p}\leq \log N$ and $\sum_{t=1}^{t_0} 1/\sqrt{t} \leq 2\sqrt{t_0}$. 
Combining \cref{eq:entropy-bound-1,eq:entropy-bound-2,eq:entropy-bound-3}, the proof is concluded.
\end{proof}

\subsection{Lower bound for OMD}
\label{sec:omd-lb-proof}

\begin{proof}[of \cref{thm:omd-lb}]
Let $q_t$ denote the probability that MW-OMD chooses the best expert (i.e., expert $\#1$) on round~$t$.
For $t \leq C$, the best expert suffers higher losses than the other expert, thus $\E[q_t] \leq 1/2$.
For $t>C$, it holds that
\begin{align*}
    q_t
    =
    \frac{e^{-\sum_{s=1}^{t-1} \eta_s (\tilde\ell_{s,1}-\tilde\ell_{s,2})}}{1+e^{-\sum_{s=1}^{t-1} \eta_s (\tilde{\ell}_{s,1}-\tilde\ell_{s,2})}}
    \leq
    e^{ -\sum_{s=1}^{t-1} \eta_s (\tilde\ell_{s,1}-\tilde\ell_{s,2}) }
    =
    e^{-\sum_{s=1}^{C} \eta_s}
    \exp\lr{ \sum_{s=C+1}^{t-1} \eta_s (\ell_{s,2}-\ell_{s,1}) }
    .
\end{align*}
Now, observe that
$$
    \sum_{s=1}^C \eta_s 
    \geq 
    C \eta_C
    =
    \alpha \sqrt{C}
    .
$$
Also,
by a standard application of Hoeffding's lemma (e.g., Appendix A of~\cite{cesa2006prediction}),
\begin{align*}
    \E\exp\lr{ \sum_{s=C+1}^{t-1} \eta_s (\ell_{s,2}-\ell_{s,1}) }
    &=
    \prod_{s=C+1}^{t-1} \E e^{\eta_s(\ell_{s,2}-\ell_{s,1})}
    \\
    &\leq
    \prod_{s=C+1}^{t-1} e^{\eta_s \Delta + \eta_s^2 / 8}
    \\
    &\leq
    \exp\lr{ \Delta \sum_{s=1}^{t-1} \eta_s } \exp\lr{ \frac18 \sum_{s=1}^{t-1} \eta_s^2 }
    \\
    &\leq
    \exp\lr{ 2\alpha\Delta\sqrt{t} + \alpha^2 \log{t} }
    .
\end{align*}
Overall, we have shown that for $t>C$,
\begin{align*}
    \E[q_t]
    \leq
    \exp\Lr{\!-\!\alpha\sqrt{C} + 2\alpha\Delta\sqrt{t} + \alpha^2 \log{t}}
    .
\end{align*}
Whenever $t \leq t_1 \eqdef \min\Lrset{2^{-6} C/\Delta^2, \exp\Lr{ \tfrac14 \sqrt{C}/\alpha}}$, the right hand side is $\leq \exp(-\tfrac12 \alpha\sqrt{C}) \leq \tfrac12$ for $\alpha \geq 1/\sqrt{C}$.
Hence, in that case,
\begin{align*}
    \cR_T
    \geq
    \sum_{s=1}^{t_1} \Delta \E[1-q_s]
    \geq 
    \sum_{s=1}^{t_1} \tfrac12 \Delta
    \geq
    \tfrac12 \Delta t_1
    &
    \qedhere
    .
\end{align*}
\end{proof}

\subsection{Analysis of OMD in the Purely Stochastic Case}
\label{sec:omd-stoch}

\begin{proof}[of \cref{thm:omd-expert-bound}]
Applying \cref{thm:omd-reg} for the experts setting we get
\begin{align*}
    \regret_T
    \leq
    \frac{1}{\eta_1} \Lr{ H(p_1) - H(p^\star) }
	+ \sum_{t=1}^{T-1} \LR{\frac{1}{\eta_{t+1}}-\frac{1}{\eta_t}} \sum_{i=1}^Np^\star_{i}\log{\frac{p^\star_{i}}{p_{t+1,i}}}
	+ \frac{1}{2} \sum_{t=1}^T \eta_t \Lr{\norm{\ell_t}_t^*}^2
    ,
\end{align*}
where we used the fact that the Bregman divergence of the negative entropy is the KL divergence.
In addition, using similar observations as in the proof of \cref{lem:ftrl-expert-regret} (e.g., \cref{eq:dual-norm,eq:geom-eta}) and setting $\eta_t=c/\sqrt{t}$ we obtain
\begin{align*}
    \regret_T
    \leq
    \frac{\log N}{c}
	+ \frac{1}{2c^2}\sum_{t=1}^{T-1} \eta_t \log{\frac{1}{p_{t+1,i^\star}}}
	+ \frac{1}{2}\sum_{t=1}^T\sum_{i=1}^N \eta_t p_{t,i}\ell_{t,i}^2
    .
\end{align*}
Applying additive translation we get,
\beq \label{eq:omd-bound}
    \regret_T
    \leq
    \frac{\log N}{c}
	+ \frac{1}{2c^2}\sum_{t=1}^{T-1} \eta_t \log{\frac{1}{p_{t+1,i^\star}}}
	+ \frac{1}{2}\sum_{t=1}^T\eta_t\sum_{i=1}^N p_{t,i}(\ell_{t,i}-\ell_{t,i^\star})^2
    .
\eeq
Similarly to \cref{lem:ftrl-2nd-order} we can bound the third term by
\beq \label{eq:omd1}
    \frac{1}{2}\sum_{t=1}^T\eta_t\sum_{i=1}^N p_{t,i}(\ell_{t,i}-\ell_{t,i^\star})^2
    \leq
    \frac{c^2}{\Delta}+\frac{1}{2}\pregret_T
    =
    \frac{\log N}{\Delta}+\frac{1}{2}\pregret_T
    .
\eeq
We now examine the second term. Using the MW algorithm defined in \cref{eq:mw-omd} we have,
\begin{align*}
    \log \frac{1}{p_{t+1,i^\star}}
    =
    \log \frac{\sum_{i=1}^{N}e^{-\sum_{s=1}^{t-1}\eta_s\ell_{t,i}}}{e^{-\sum_{s=1}^{t-1}\eta_s\ell_{s,i^\star}}} 
    = 
    \log\LR{ 1+\sum_{i\neq i^\star} e^{-\sum_{s=1}^{t-1}\eta_s(\ell_{s,i} - \ell_{s,i^\star} )}}
    .
\end{align*}
Plugging it back to the original term we get
\begin{align*}
    \frac{1}{2c^2}\sum_{t=1}^{T-1} \eta_t \log{\frac{1}{p_{t+1,i^\star}}}
    =
    \frac{1}{2c}\sum_{t=1}^{T-1}\frac{1}{\sqrt{t}}\log\LR{ 1+\sum_{i\neq i^\star} e^{-\sum_{s=1}^{t-1}\eta_s(\ell_{s,i} - \ell_{s,i^\star})} }
    .
\end{align*}
By taking the expectation and using its linearity property we obtain
\begin{align*}
    \frac{1}{2c}\sum_{t=1}^{T-1}\frac{1}{\sqrt{t}}\E \LRbra{\log\LR{ 1+\sum_{i\neq i^\star} e^{-\sum_{s=1}^{t-1}\eta_s(\ell_{t,i} - \ell_{t,i^\star})} }}
    &\leq
    \frac{1}{2c}\sum_{t=1}^{T-1}\frac{1}{\sqrt{t}}\log\LR{ 1+ \sum_{i\neq i^\star} \E\Lrbra{ e^{-\sum_{s=1}^{t-1}\eta_s(\ell_{s,i} - \ell_{s,i^\star})}}} \\
    &\leq 
    \frac{1}{2c}\sum_{t=1}^{T-1}\frac{1}{\sqrt{t}}\log{\LR{1 +\sum_{i\neq i^\star}\prod_{s=1}^{t-1} \E\Lrbra{ e^{-\eta_s(\ell_{s,i} - \ell_{s,i^\star})}}}}
    ,
\end{align*}
where we used Jensen inequality for concave functions for the first inequality and the fact that $x_t\eqdef\ell_{t,i}-\ell_{t,i^\star}$ are i.i.d.~for the second inequality. 
Applying Hoeffding's Lemma yields,
\begin{align*}
    \frac{1}{2c}\sum_{t=1}^{T-1}\frac{1}{\sqrt{t}}\log{\LR{1 +\sum_{i\neq i^\star}\prod_{s=1}^{t-1} \E\Lrbra{ e^{-\eta_s(\ell_{s,i} - \ell_{s,i^\star})}}}}
    \leq
    \frac{1}{2c}\sum_{t=1}^{T-1}\frac{1}{\sqrt{t}}\log \lr{ 1 + N \exp\LR{\sum_{s=1}^{t-1}\Lr{\tfrac12 \eta_s^2 - \eta_s\Delta}} }
    .
\end{align*}
Next, we bound the argument of the exponent
\begin{align*}
    \sum_{s=1}^{t-1} \Lr{ \tfrac12 \eta_s^2 - \eta_s\Delta }
    &\leq
    \frac{c^2}{2}\sum_{s=1}^{t-1}\frac{1}{s} 
    - c\Delta\sum_{s=1}^{t-1}\frac{1}{\sqrt{s}} \\
    &\leq
    \frac{c^2}{2}(1+\log t)
    - c\Delta\sqrt{t} \\
    &\leq
    c^2\log t 
    - c\Delta\sqrt{t}
    ,
\end{align*}
where we bounded the summations by their integrals. 
Therefore we have
\begin{align*}
    \frac{1}{2c}\sum_{t=1}^{T-1}\frac{1}{\sqrt{t}}\log \LR{ 1 + Ne^{\sum_{s=1}^{t-1}(\frac{\eta_s^2}{2} - \eta_s\Delta)} }
    &\leq
    \frac{1}{2c}\sum_{t=1}^{T-1}\frac{1}{\sqrt{t}}\log \LR{ 1 + Ne^{c^2\log t
    - c\Delta\sqrt{t}} }
    .
\end{align*}
First we examine the sum from $t_1$ onward, while we require that for $t\geq t_1$ it holds
\beq\label{eq:omd-t0}
c^2\log t \leq \tfrac{1}{2} c\Delta \sqrt{t}
.
\eeq
To satisfy \cref{eq:omd-t0} it suffices to take
$$
    t_1
    =
    \LR{\frac{8c}{\Delta}}^2 \log^2\frac{8c}{\Delta}
    .
$$
Therefore,
\begin{align*}
    \frac{1}{2c}\sum_{t=t_1+1}^{T-1}\frac{1}{\sqrt{t}}\log \LR{ 1 + Ne^{c^2\log t - 2c\Delta\sqrt{t}} }
    &\leq
    \frac{1}{2c}\sum_{t=t_1+1}^{T-1}\frac{1}{\sqrt{t}}\log \LR{ 1 + Ne^{ -\frac{1}{2}c\Delta\sqrt{t}} } \\
    &\leq
    \frac{N}{2c}\sum_{t=t_1+1}^{T-1}\frac{1}{\sqrt{t}}e^{ -\frac{1}{2}c\Delta\sqrt{t}}
    \tag{$\log (1+x)\leq x$} \\
    &\leq
    \frac{N}{2c}\int_{t_1}^{T-1}\frac{1}{\sqrt{t}}e^{ -\frac{1}{2}c\Delta\sqrt{t}}dt \\
    &\leq
    \frac{N}{c^2\Delta}e^{ -\frac{1}{2}c\Delta\sqrt{t_1}}dt\\
    &\leq
    \frac{2N}{c^2\Delta}e^{ -c^2\log{t_1}} \tag{$c^2\log t_1 \leq \frac{1}{2}c\Delta \sqrt{t_1}$}\\
    &\leq
    \frac{2}{\Delta \log{N}}
    \tag{$ t_1 > 1$ and $c=\sqrt{\log N}$}
    .
\end{align*}
To conclude we examine the bound up to $t_1$,
\begin{align*}
    \frac{1}{2c}\sum_{t=1}^{t_1}\frac{1}{\sqrt{t}}\log \LR{ 1 + Ne^{c^2\log t - 2c\Delta\sqrt{t}} }
    &\leq
    \frac{1}{2c}\sum_{t=1}^{t_1}\frac{1}{\sqrt{t}}\log \LR{ 2Ne^{c^2\log t
    - 2c\Delta\sqrt{t}} } \\
    &\leq
    \frac{1}{2c}\sum_{t=1}^{t_1}\frac{1}{\sqrt{t}}\LR{\log 2N + c^2\log t} \\
    &\leq
    \frac{\log 2N + c^2\log t_1}{2c}\sum_{t=1}^{t_1}\frac{1}{\sqrt{t}}
    \tag{$\log t \leq \log t_1$} \\
    &\leq
    \frac{\log 2N + c^2\log t_1}{c}\sqrt{t_1}
    .
\end{align*}
Since $c=\sqrt{\log N}$, for $t_1\geq \ceil{e^2}$ we have
$
    c^2\log t_1=\log N \log t_1
    \geq 
    \log 2N
    ,
$
and also \cref{eq:omd-t0} still holds. This implies
\begin{align*}
    \frac{\log 2N + c^2\log t_1}{c}\sqrt{t_1}
    &\leq
    2c\log t_1\sqrt{t_1} \\
    &\leq
    2\Delta t_1 \\
    &\leq
    128\frac{\log N}{\Delta}\log^2\LR{\frac{8\sqrt{\log N}}{\Delta}}
    .
\end{align*}
when we used the fact that $c\log t_1 \leq \Delta \sqrt{t_1}$ for the last inequality.
Adding both results(up to $t_1$ and from $t_1$ onward) we obtain,
\beq \label{eq:omd2}
    \frac{1}{2c^2}\sum_{t=1}^{T-1} \eta_t \log{\frac{1}{p_{t+1,i^\star}}}
    \leq
    128\frac{\log N}{\Delta}\log^2\LR{\frac{8\sqrt{\log N}}{\Delta}} 
    + \frac{2}{\Delta\log N}
\eeq
Finally, plugging \cref{eq:omd1,eq:omd2} into \cref{eq:omd-bound}, taking the expectation and rearranging terms we get
\begin{align*}
    \E \Lrbra{\pregret_T} 
    \leq 
    \frac{256\log N}{\Delta}\log^2\LR{\frac{8\log N}{\Delta}} 
    + \frac{8\log N}{\Delta}
    .
    &\qedhere
\end{align*}

\end{proof}

\newpage

\section*{Acknowledgments}

We thank Alon Cohen for helpful discussions.
This project has received funding from the European Research Council (ERC) under the European Union’s Horizon 2020 research and innovation program (grant agreement No. 882396), from the Israel Science Foundation (grants 2549/19; 993/17, 2188/20), from the Yandex Initiative in Machine Learning and partially funded by an unrestricted gift from Google. Any opinions, findings, and conclusions or recommendations expressed in this work are those of the author(s) and do not necessarily reflect the views of Google.

\bibliography{paperbib}

\newpage
\appendix

\section{Analysis of Time-varying Regularization Algorithms}
\label{sec:online-opt-analysis}

In this section, we assume the setup of online (linear) optimization, with the notation established in \cref{sec:online-opt}.
For the proofs below, we recall the notion of a Bregman divergence.
For a continuously differentiable and strictly convex function $F:\cW \rightarrow \reals$ defined on a closed convex set $\cW$, the Bregman divergence associated with $F$ at a point $w \in \cW$ is defined by
\begin{align*}
    \forall w' \in \cW,\qquad
    D_F(w',w) = F(w')-F(w) -\nabla F(w)\cdot(w'-w)
    .
\end{align*}

\subsection{Follow the Regularized Leader}
\label{app:ftrl-analysis}

First, we present a general analysis for Follow the Regularized Leader, described in \cref{eq:ftrl-def}, and later establish \cref{thm:ftrl-rt}.

\begin{theorem} \label{lem:rftl2}
There exists a sequence of points $z_t \in [w_t,w_{t+1}]$ such that, for all $w^\star \in \cW$,
\begin{align*}
	\sum_{t=1}^T \grad_t \cdot (w_t - w^\star)
	\leq
	R_{T+1}(w^\star) - R_1(w_1)
	+ \sum_{t=1}^T \Lr{ R_t(w_{t+1}) - R_{t+1}(w_{t+1}) }
	+ \frac{1}{2} \sum_{t=1}^T \Lr{\norm{\grad_t}_t^*}^2
	.
\end{align*}
Here $\norm{w}_t = \sqrt{w\tr \nabla^2 R_t(z_t) w}$ is the local norm induced by $R_t$ at $z_t$, and $\norm{\cdot}_t^*$ is its dual.
\end{theorem}

\begin{proof}
Denote $\Phi_t(w) = w \cdot \sum_{s=1}^{t-1} \grad_s + R_{t}(w)$, so that $w_t = \argmin_{w \in \cW} \Phi_t(w)$.
We first write
\begin{align*}
	\sum_{t=1}^T \grad_t \cdot w_{t+1}
	&=
	\sum_{t=1}^T \Lr{ \Phi_{t+1}(w_{t+1}) - \Phi_t(w_{t+1}) }
	+ \sum_{t=1}^T \Lr{ R_t(w_{t+1}) - R_{t+1}(w_{t+1}) }
	\\
	&=
	\Phi_{T+1}(w_{T+1}) - \Phi_1(w_1)
	+ \sum_{t=1}^T \Lr{ \Phi_t(w_t) - \Phi_t(w_{t+1}) }
	+ \sum_{t=1}^T \Lr{ R_t(w_{t+1}) - R_{t+1}(w_{t+1}) }
	.
\end{align*}
Since $w_t$ is the minimizer of $\Phi_t$ over $\cW$, first-order optimality conditions imply
\begin{align*}
	\Phi_t(w_t) - \Phi_t(w_{t+1})
	=
	- \nabla\Phi_t(w_t) \cdot (w_{t+1}-w_t) - D_{\Phi_t}(w_{t+1},w_t)
	\leq
	- D_{\Phi_t}(w_{t+1},w_t)
	=
	- D_{R_t}(w_{t+1},w_t)
	,
\end{align*}
where we have used the fact that the Bregman divergence is invariant to linear terms.
On the other hand, since $w_{T+1}$ is the minimizer of $\Phi_{T+1}$, we have that
\begin{align*}
	\sum_{t=1}^T \grad_t \cdot w^\star
	=
	\Phi_{T+1}(w^\star) - R_{T+1}(w^\star)
	\geq
	\Phi_{T+1}(w_{T+1}) - R_{T+1}(w^\star)
	.
\end{align*}
Combining inequalities and observing that $\Phi_1(w_1) = R_1(w_1)$, we obtain
\begin{align*}
	\sum_{t=1}^T \grad_t \cdot (w_{t+1} - w^\star)
	\leq
	R_{T+1}(w^\star) - R_1(w_1)
	+ \sum_{t=1}^T \Lr{ R_t(w_{t+1}) - R_{t+1}(w_{t+1}) }
	- \sum_{t=1}^T D_{R_t}(w_{t+1},w_t)
	.
\end{align*}
On the other hand, a Taylor expansion of $R_t(\cdot)$ around $w_t$ with an explicit second-order remainder term implies that, for some intermediate point $z_t \in [w_t,w_{t+1}]$, it holds that
\begin{align*}
	D_{R_t}(w_{t+1},w_t)
	=
	\thalf (w_{t+1} - w_t)\tr \, \hess R_t(z_t) \, (w_{t+1} - w_t)
	=
	\tfrac{1}{2} \norm{w_{t+1}-w_t}_t^2
	.
\end{align*}
An application of Holder's inequality then gives
\begin{align*}
	\grad_t \cdot (w_t - w_{t+1})
	\leq
	\norm{\grad_t}_t^* \, \norm{w_t - w_{t+1}}_t
	\leq
	\tfrac{1}{2}\Lr{\norm{\grad_t}_t^*}^2 + \tfrac{1}{2} \norm{w_t - w_{t+1}}_t^2
	=
	\tfrac{1}{2}\Lr{\norm{\grad_t}_t^*}^2 + D_{R_t}(w_{t+1},w_t)
	.
\end{align*}
The proof is finalized by summing over $t=1,\ldots,T$ and adding to the inequality above.
\end{proof}

\begin{proof} [of \cref{thm:ftrl-rt}]
Fix any $w^\star \in \cW$.
Observe that FTRL with regularizations $R_t(w) = \eta_t^{-1} R(w)$ is equivalent to FTRL with $R_t(w) = \eta_t^{-1} (R(w)-R(w^\star))$.
Applying \cref{lem:rftl2} for the latter and rearranging, we obtain the claimed bound.
\end{proof}

\subsection{Online Mirror Descent}
\label{app:omd-analysis}

We next consider Online Mirror Descent (see \cref{eq:omd-def}), and prove the following general bound from which \cref{thm:omd-reg} directly follows.

\begin{lemma} \label{lem:omd}
There exist points $z_t \in [w_t,w'_{t+1}]$ such that for all $w^\st \in \cW$,
\begin{align*}
\sum_{t=1}^T \grad_t \!\cdot\! (w_t - w^\st)
&\leq
R_1(w^\st) - R_1(w_1) +
\sum_{t=1}^{T-1} \Lr{D_{R_{t+1}}(w^\st,w_{t+1}) - D_{R_t}(w^\st,w_{t+1})}
+ \frac{1}{2} \sum_{t=1}^T \Lr{\norm{\grad_t}_{t}^*}^2
\!.
\end{align*}
Here $\norm{w}_t = \sqrt{w\tr \hess R_t(z_t) w}$ is the local norm induced by $R_t$ at $z_t$, and $\norm{\cdot}_t^*$ is its dual.
\end{lemma}

\begin{proof}
Fix any $w^\st \in \cW$.
We will bound each of the terms $\grad_t \cdot (w_t - w^\st)$.
First, from the update rule of Mirror Descent and the three-point property of the Bregman divergence, we have
\begin{align*}
\grad_t \cdot (w'_{t+1} - w^\st)
&=
(\nabla R(w_t) - \nabla R(w'_{t+1})) \cdot (w'_{t+1} - w^\st)
\\
&=
D_{R_t}(w^\st,w_t) - D_{R_t}(w^\st,w'_{t+1}) - D_{R_t}(w'_{t+1},w_t)
.
\end{align*}
Now, a Taylor expansion of $R_t$ at $x_t$ (with an explicit Lagrange remainder term) shows that there exists $z_t \in [w_t,w_{t+1}]$ for which
\begin{align*}
D_{R_t}(w'_{t+1},w_t)
=
\thalf (w'_{t+1} - w_t)\tr \, \hess R_t(z_t) \, (w'_{t+1} - w_t)
=
\thalf \norm{w'_{t+1} - w_t}_{t}^{2}
.
\end{align*}
Also, since $w_{t+1}$ is the projection (with respect to the Bregman divergence $R_t$) of the point $w'_{t+1}$ onto the set $\cW$ that contains $w^\st$, it holds that $D_{R_t}(w^\st,w_{t+1}) \le D_{R_t}(x^\st,w'_{t+1})$.
Putting things together, we obtain
\begin{align} \label{eq:gt1}
\grad_t \cdot (w'_{t+1} - w^\st)
\le
D_{R_t}(w^\st,w_t) - D_{R_t}(w^\st,w_{t+1}) - \thalf \norm{w'_{t+1} - w_t}_{z_t}^{2}
.
\end{align}
On the other hand, H\"older's inequality and the fact that $ab \le \half (a^2 + b^2)$ yield
\begin{align} \label{eq:gt2}
\grad_t \cdot (w_t - w'_{t+1})
\le
\norm{\grad_t}_{t}^* \cdot \norm{w_t - w'_{t+1}}_{t}
\le
\thalf (\norm{\grad_t}_{t}^*)^2 + \thalf \norm{w_t - w'_{t+1}}_{t}^2
.
\end{align}
Summing \cref{eq:gt1,eq:gt2} together over $t=1,\ldots,T$ gives the regret bound
\begin{align*}
\sum_{t=1}^T \grad_t \cdot (w_t - w^\st)
&\leq
\sum_{t=1}^T \Lr{D_{R_t}(w^\st,w_t) - D_{R_t}(w^\st,w_{t+1})}
+ \frac{1}{2} \sum_{t=1}^T \Lr{\norm{\grad_t}_{t}^*}^2
.
\end{align*}
Rearranging the first summation and using the facts that $D_{R_T}(w^\st,w_{T+1}) \geq 0$ and $D_{R_1}(w^\st,w_1) \leq R_1(w^\st) - R_1(w_1)$ (the latter follows since $w_1$ is the minimizer of $R_1$, and so $\nabla R_1(w_1)\cdot (w^\st-w_1) \geq 0$) gives the stated regret bound.
\end{proof}

\end{document}